\documentclass[preprint,12pt,authoryear]{elsarticle}
\usepackage[utf8]{inputenc} 
\usepackage{hyperref}       
\usepackage{url}            
\usepackage{booktabs}       
\usepackage{amsfonts}       
\usepackage{amsmath}
\usepackage{amsthm}
\usepackage{geometry}
\usepackage{nicefrac}       
\usepackage{microtype}      
\usepackage{xcolor}         
\usepackage{svg}
\usepackage{xcolor,listings}
\usepackage{syntax}
\usepackage{bcprules}
\usepackage{tikz}
\usepackage{multicol}
\usetikzlibrary{arrows.meta}
\usepackage{graphicx}
\usepackage{appendix}
\usepackage{pgfplots}
\usepackage{arydshln}
\usepackage{subcaption}
\usepackage{algorithm}
\usepackage{algpseudocode}
\usepackage{pdfpages}
\usepackage{mathrsfs}
\usetikzlibrary{arrows}
\usetikzlibrary[patterns]
\usetikzlibrary{shapes.geometric}
\usetikzlibrary{decorations.pathreplacing,matrix}
\tikzset{ 
table/.style={
  matrix of math nodes,
  row sep=-\pgflinewidth,
  column sep=-\pgflinewidth,
  nodes={rectangle,text width=1em,align=center},
  text depth=1.ex,
  text height=1ex,
  nodes in empty cells,
  left delimiter=[,
  right delimiter={]},
  ampersand replacement=\&
}
}

\usepackage{babel}
\usepackage{csquotes}
\usepackage{hyperref}
\usepackage[round]{natbib}
\newtheorem{remark}{Remark}

\newtheorem{theorem}{Theorem}[section]

\newcommand{\argmin}{\ensuremath{arg\,min}}

\newcommand{\norm}[1]{\left\lVert#1\right\rVert}
\newcommand{\abs}[1]{\mid #1 \mid}

\newcommand{\iif}{ \Leftrightarrow}

\newcommand{\summ}[2]{\underset{#1}{\overset{#2}{\sum}}}
\newcommand{\timess}[2]{\underset{#1}{\overset{#2}{\times}}}
\newcommand{\esp}[1]{\mathbb{E}[#1]}
\newcommand\numberthis{\addtocounter{equation}{1}\tag{\theequation}}
\newcommand{\pdv}[2]{\frac{\partial #1}{\partial #2}}

\newcommand{\reel}{\mathbb{R}}

\newcommand{\tecnameAbrv}{GCE}

\usepackage{xcolor, soul}
\definecolor{lightblue}{rgb}{0.9, 0.9, 1.0}
\definecolor{ffqqff}{rgb}{1,0,1}
\sethlcolor{lightblue} 

\begin{document}
\begin{frontmatter}
\title{Stochastic gradient descent with gradient estimator for categorical features}
\author[Lokad,Litis]{Paul Peseux}
\ead{paul.peseux@gmail.com}
\author[add2]{Thierry Paquet}
\ead{thierry.Paquet@univ-rouen.fr}
\author[add2]{Maxime Berar}
\ead{maxime.berar@univ-rouen.fr }
\author[add1]{Victor Nicollet}
\ead{victor.nicollet@lokad.com}
\address[Lokad]{Lokad, Paris FRANCE}
\address[Litis]{Litis, Rouen FRANCE}
\begin{abstract}
Categorical data are present in key areas such as health or supply chain, and this data require specific treatment. In order to apply recent machine learning models on such data, encoding is needed. To build interpretable models, one-hot encoding is still a very good solution, but such encoding creates sparse data. Gradient estimators are not suited for sparse data: the gradient is mainly considered as zero while it simply does not always exist, thus a novel gradient estimator is introduced. We show its efficiency on different datasets with varied model architectures. This new estimator performs better than common estimators under similar settings. A real world retail dataset is also released. Overall, this paper aims to thoroughly consider categorical data and adapt models and optimizers to these key features.
\end{abstract}

\end{frontmatter}

\section{Introduction}\label{sec:intro}

Tabular data represents a considerable amount of modern data especially in the healthcare and industrial sectors. Machine Learning has been applied to those tabular data for decades for different tasks such as regression or classification. Boosting methods \cite{Chen2016XGBoostAS, Ostroumova2018CatBoostUB} are widely spread on these data and are still state of the art. After outstanding results on image, speech recognition or text, some attempts to apply deep learning (DL) models on tabular data have been published recently but with a limited impact on the state of the art as presented by \cite{RevisitingDeepForTabular}. Some DL approaches \cite{NODE, Frosst2017DistillingAN, Luo2021SDTRSD} have tried to adapt their architecture  to the specificity of tabular data but none did succeed in overtaking classical machine learning models such as gradient-boosted tree ensembles \cite{DeepTabularSurvey}. 

One of the possible explanation is that DL excels on homogeneous data where embeddings can be learned \cite{MLPandNLP, embeddingSurvey} via stochastic gradient descent \cite{firstSGD} to update their parameters by utilizing the underlying nature of the data like 2D spatial pixel neighborhood. In contrast, there is no such general underlying structure on tabular data. As tabular data often contains categorical data which are not numerical, the inputs of the model consists in the encoding of data. When the categorical data cardinality is limited, one-hot encoding is a good solution. Beyond its simplicity, it creates category-related parameters that are key for model explainability. In this encoding, only one feature is set to $1$ (indicating the presence of the corresponding category), while all other features are set to $0$. During the model optimization through gradient descent, the gradient is only present for the active feature, and all other features have zero gradient. The issue with this is that a zero gradient can halt the optimization process, as the optimizer updates all the parameters, including the ones corresponding to the inactive features. This might explain the underperforming results of DL models on categorical data. This observation applies also on non-deep models whose training rely on gradient descent. We investigate this issue by directing our attention towards the categorical aspect of the data, which is the root cause of the problem, rather than model architectures.

Our main contributions are the following. First, we propose a modification of the standard training loss on categorical data with a related unbiased estimator. Second, we  show that this new gradient estimator outperforms strandard ones on different datasets. Third, we applied this technique to an in-production model using a private dataset that we are releasing as open source for this purpose. This dataset is substantial and pertains to the supply chain, with only a few publicly available datasets in this domain.

The article is organized as follows. After an overview of the recent works on tabular data we point out the issue of applying modern stochastic gradient descent on one-hot-encoded categorical data. Then we propose a novel gradient estimator and show that it is unbiased for a relevant loss on categorical data. We conducted several experiments on public and private datasets that demonstrate the superiority of our proposed gradient estimator over the classical gradient estimator. These results encourage the use of our estimator in the presence of categorical data.

\section{Learning with categorical data}
\subsection{Related works}\label{sec:relatedWorks}

As described in \cite{DeepTabularSurvey} and \cite{RevisitingDeepForTabular}, tabular data exhibit heterogeneity, characterized by high variability of data types and formats, in their underlying structure, unlike images. They involve categorical input attributes and have a strong structure that is unique to each tabular dataset. Modifying a categorical attribute in the input may lead to a complete change in the meaning of the corresponding data, whereas changing a pixel in an image does not fundamentally alter the image. This data kind distinction can lead to different results for the same deep learning architectures, as shown in the case of adversarial learning \cite{MATHOV2022108377}. Even for simpler tasks such as binary or multiclass classification and regression, DL did not surpass yet tree models on tabular data yet as presented in \cite{DLisnotyouneed, DeepTabularSurvey}. Tabular data is depicted as the last “unconquered castle” by \cite{Kadra2021WelltunedSN} and multiple works asses the crucial need of further development in this direction \cite{Fayaz2022}. This holds even though various architectures such as MLP, ResNet, Transformer, NET-DNF \dots have been applied to them.

We can split the architectures  into two categories: the raw DL models and the adapted DL models. The first rely on some known DL models directly applied on tabular data, without any modification of their architecture. One example is the work presented in \cite{DeepTLF}, which attempts to transform the heterogeneous nature of tabular data into a homogeneous numerical representation, in order to apply successful DL methods to this type of data. The second one adapts DL architectures in order to better fit the tabular data specificity \cite{Arik2021TabNetAI, Song2019AutoIntAF, NODE}. 

All these attempts did not outperform the standards models such as XGBoost from \cite{Chen2016XGBoostAS} or CatBoost from \cite{Ostroumova2018CatBoostUB} which are still the state-of-the-art in this domain \cite{DeepTabularSurvey}, as shown on the Adult Census Income (ACI) dataset \cite{incomeDataset}. This dataset is frequently utilized to showcase the handling of categorical data. DL approaches assume that every input feature is relevant to every observation, which can explain their disappointing results on categorical data. In DL architectures, all parameters are typically updated at every iteration except through the use of methods such as Dropout \cite{dropout} or LayerOut \cite{freezing} that are general learning tricks non specific to any kind of data. This assumption may not hold true for categorical data, leading to potential inaccuracies in the training process. 

Hence, there is a requirement for the development of novel approaches that can better account for the inherent characteristics of categorical data and are better equipped to handle the characteristics associated with this type of data.

\subsection{Categorical models and one-hot-encoding}

\small
\begin{table}[h!]
  \caption{Categorical data.}
  \label{tab:catData}
  \begin{footnotesize}
  \begin{center}
  \begin{tabular}{llcc}
    \toprule
    Id & Cat  & Discount ($\%$) & Sales \\
    \midrule                         
    001 & pants     & 20       &  7               \\
    002 & shirt    & 10       &  3                \\
    003 & shirt    & 15       &  2                \\
    \dots & \dots   & \dots    &  \dots             \\
    n & shirt    & 20       &  8                 \\
  \bottomrule
\end{tabular}
\end{center}
\end{footnotesize}
\end{table}
\normalsize

DL methods are designed to work well with numerical data, such as arrays of continuous values, because they are built on mathematical operations that are well-defined for numerical data. Categorical data, on the other hand, refers to data that can take on a limited number of discrete values, and there is no existing ordering of the value. Let us denote \textit{categorical models} the set of models that accept categorical attributes by design and are numerical , i.e. their parameters can be updated through gradient descent. By categorical attributes, we denote an attribute whose possible values belongs to an alphabet of $n_{s}$ symbols $\left\{s_1,\cdots, s_{n_{s}} \right\}$. In Table \ref{tab:catData}, \textit{Cat} is the only  categorical attribute, while pants and shirt are the symbols, and forms the \textit{Cat} alphabet. We stress that the \textit{categorical} aspect applies to the nature of the input attributes: a categorical model can be used for regression. In that sense regular neural networks are not categorical models as they need numerical inputs. Some specific deep models correspond to this definition as wide models described in \cite{wideAndDeep}. 

We present a simple example of a categorical model, which is the relational linear regression. A traditional linear regression predicts a numerical quantity using a numerical input and two parameters, namely the slope and the intercept. The relational linear regression extends it with the slope being shared among all categories of a categorical attribute. The intercept remains shared among all observations.
\newpage

\begin{multicols}{2}
  \begin{equation}\label{catModel}
    \hat{y}(cat, x) = a_{cat}  \times x + b
\end{equation}\break
        \centering
        \includegraphics[width=0.45\textwidth]{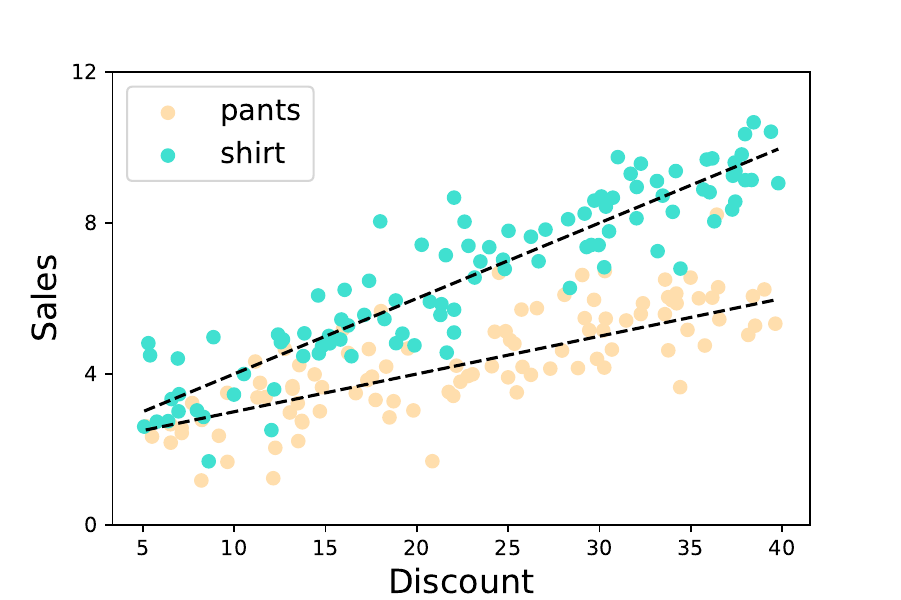}
        \captionof{figure}{Relational linear regression.}
        \label{fig:toydata}
\end{multicols}

A linear regression has $2$ parameters, while a relational linear regression has $1 + n_s$ parameters, with $n_s$ the cardinality of the categorical attribute. Instead of only modeling the relationship between two variables, it utilizes the underlying structure of the data to provide a more accurate representation of the attributes relationship.

Let us apply the categorical model to the inputs presented in Table \ref{tab:catData}, with the goal of predicting sales based on the discount and the category of the item. This application is visualized in the graphical representation shown in Figures \ref{fig:toydata} and \ref{fig:featTokenizerOWN}.

\begin{figure}[h!]
\centering
\begin{tikzpicture}[scale=0.9]
\begin{scope}[every node/.style={thick,draw,minimum size=1cm,fill=green!50,opacity=.2,text opacity=1}]
    \node (Color) at (1,0.3)  {pants};
\end{scope}

\begin{scope}[every node/.style={thick,draw,minimum size=1cm,opacity=.2,text opacity=1}]
    \node (Discount) at (3.5,0.3) {\small 20 $\%$};
\end{scope}

\begin{scope}
    \node (mucolor)  at (1.0,2.8) {$a_{cat}$};
    \node (times)  at (2.75,3.1) {$\times$};
    \node (plus)  at (3.5,3.1) {$+$};
\end{scope}

\begin{scope}
    \node (dataColor)  at (1.0,-0.6) {Cat};
    \node (data)  at (-2.9,0.3)    [anchor=west]  {\small data};
    \node (param)  at (-2.9,2)  [anchor=west]    {\small parameters};
    \node (pred)  at (-2.9,4)   [anchor=west]   {\small prediction};
    \node (dataStore)  at (3.5,-0.6)    {Discount};
\end{scope}

\begin{scope}[every node/.style={thick,draw,minimum size=1cm,fill=green!50,opacity=.8,text opacity=1}]
    \node (pink) at (1.8,2)    {\textcolor{black}{$a_{pants}$}};
\end{scope}

\begin{scope}[every node/.style={thick,draw,minimum size=1cm,fill=green!50,opacity=.2,text opacity=1}]
    \node (blue) at (0.45,2)  {\textcolor{black}{$a_{shirt}$}};
\end{scope}

\begin{scope}[every node/.style={thick,draw,minimum size=1cm,fill=yellow!50,opacity=.8,text opacity=1}]
    \node (M) at (5,2) {$b$};
\end{scope}

\begin{scope}[every node/.style={thick,draw,minimum size=1cm,fill=gray!50,opacity=.2,text opacity=1}]
    \node (y) at (3,4) {$\hat{y}$};
\end{scope}
\begin{scope}[>={Stealth[black]},
              every node/.style={fill=white,circle},
              every edge/.style={thin, color=black}]
    \path [->] (Color) edge[draw=black] (pink);
    \path [->] (Discount)  edge[draw=black] (y);
    \path [->] (pink)  edge[draw=black] (y);
    \path [->] (M)     edge[draw=black] (y);
    
\end{scope}
\end{tikzpicture}
\caption{Relational linear regression accessing parameters. The slope parameters are not concerned by every observation: parameter $a_{pants}$ is used only for the pants data.}
\label{fig:featTokenizerOWN}
\end{figure}
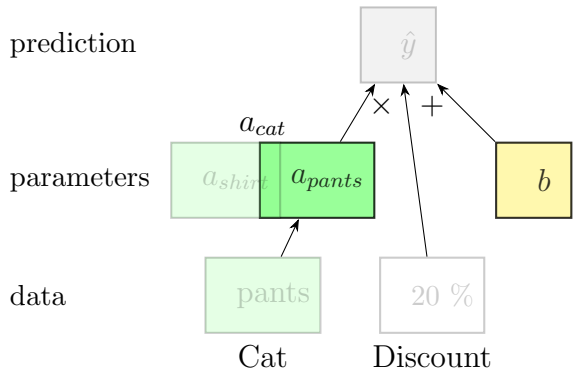

In this application, the parameter $a_{cat}$ has a value for each possible category of \textit{Cat}, and we aim to find the best ones, with the appropriate intercept, in order to build a good predictive model. One of the primary methods to accomplish that is gradient descent. Partly due to the very large amount of data often encountered in practice, \textit{stochastic} gradient descent is used. To apply stochastic gradient descent on categorical models, the categorical data has to be encoded into numerical features. No universally good method of encoding exists and encoding choice should always rely on data (alphabet cardinality, relationships between them \dots). Our focus will be on one-hot encoding, as categorical models rely on this mechanism. One-hot-encoding a categorical variable with cardinality $n$ is performed by creating a $n$ dimensional binary vector. If there are few symbols, there are only a few newly created columns. For example, on data stored in Table \ref{tab:catData}, one-hot-encoding the attribute \textit{Cat} creates the features $is_{pants}$ and $is_{shirt}$.

In high-stakes contexts such as disease diagnostics, interpretability of the model is fundamental  \cite{Stop}. In such scenarios, the human expert is expected to make the final decision based on the explainability of the model's results. Tree-based models are known for their interpretability and have been used to interpret deep learning models \cite{BLANCOJUSTICIA2020105532}. In this direction, using one-hot-encoding  is crucial. Having parameters directly related to the application semantic by giving access to their relation with the input symbols is a requirement for the design of white-box models. In the illustrated example, the variable $a_{pants}$ holds significant meaning, namely the degree of sensitivity of sales in the pants category to the proposed discount. Parameter values not only serve model prediction quality, they are also \textit{interpretable}.  On Model \ref{catModel}, $a_{pants} >a_{shirt}$ means that the pants sales better react to the discount than the shirt ones. Not only is the prediction of the model explainable, but the trained model itself conveys meaning because the parameters have their own semantic. It also maintains the structure of the data: it does not impose any arbitrary ordering on the nominal categories (no intrinsic order).

When dealing with low cardinality attributes, one-hot-encoding is a suitable approach to turn them into numerical values. However, if this approach is used for high cardinality attributes, the curse of dimensionality may arise, as explained in \cite{curse}. In such cases, alternative encoding methods should be considered. The leave-one-out encoding method transforms a categorical attribute into a numerical feature, offering several benefits such as avoiding the curse of dimensionality. However, this approach does not result in interpretable parameters. For the purposes of this article, we will exclude such high-cardinality categorical attributes, which are not common in domains such as health or supply chain. We also exclude any attribute that has a native ordinal encoding, like size attribute with possible values $\{$ small ;  medium ; large$\}$.

Applying stochastic gradient descent on categorical models raises an issue as common gradient update techniques are not designed for one-hot encoded categorical features: not every symbol of a categorical attribute is present in every observation of a dataset while regular numerical models assume that every feature is present on every observation. Thus we propose an updated version of gradient estimation used to update categorical parameters. Its specificity is to take into account the categorical features of the model. 

\subsection{Problem and one-hot-encoding notations}\label{notations}

We consider the supervised learning set up with a given set of training labeled dataset $\mathcal{Z} = \{ z_i = (\textbf{C}_i; y_i); i = 1 \dots  n \}$, with the attribute vectors $\textbf{C}_i \in \timess{c = 1}{C} A_c$ where each $A_c$ is an alphabet, i.e. a finite set of $\abs{A_c}$ symbols. Thanks to one-hot encoding, one can turn the attribute vectors $\textbf{C}_i$ into numerical features $X_i \in \{0 ; 1\}^p$ where $p = \summ{c=1}{C} \abs{A_c}$. Let's define an arbitrary order on the union of all the alphabets: $\{ s_k\}_{k \leq p} = \underset{c}{\bigcup} A_c$. Then:
\begin{equation} \label{eq:symbol}
    \forall (C_i, \cdot) \in \mathcal{Z} \quad \forall k \leq p \quad \textbf{C}_i = s_k \iif X_{i}^k = 1
\end{equation}

\begin{figure*}
\centering
\begin{subfigure}{0.48\linewidth}
\centering
\begin{tikzpicture}[scale=0.65]

\draw (-0.9,10.5) --(-1,10.5) -- (-1,2.2) -- (-0.9,2.2) ;
\draw (0.9,10.5) --(1,10.5) -- (1,2.2) -- (0.9,2.2) ;

\node (x1)   at (0,10)  {$x_1$};
\node (x2)   at (0,9)  {$x_2$};
\node (x3)   at (0,8)  {$x_3$};
\node (x4)   at (0,7)  {$x_4$};
\node (x5)   at (0,6)  {$x_5$};
\node (xdot) at (0,5)  {\dots};
\node (xpm1) at (0,3.9)  {$x_{p-1}$};
\node (xp)   at (0,2.9)  {$x_p$};

\node (feat1) at (-2,9)   {\textcolor{orange}{$cat_1$}};
\node (feat2) at (-2,6.5) {\textcolor{black}{$cat_2$}};
\node (feat3) at (-2,3.4) {\textcolor{green}{$cat_C$}};

\node (X)     at (-1.9,5) {\Large $X = $};
\node (Theta) at (5.7,5)  {\Large $ = \theta $};

\draw [decorate, decoration = {brace,mirror}, color=orange] (-0.6,10.2) -- (-0.6,7.8);
\draw [decorate, decoration = {brace,mirror}, color=black]  (-0.6,7.2)  -- (-0.6,5.8);
\draw [decorate, decoration = {brace,mirror}, color=green]  (-0.6,4.1)  -- (-0.6,2.7);

\draw (3.1,10.5) --(3,10.5) -- (3,1.7) -- (3.1,1.7) ;
\draw (4.9,10.5) --(5,10.5) -- (5,1.7) -- (4.9,1.7) ;

\node (theta1)   at (4,10)   {$\theta_1$};
\node (theta2)   at (4,9)    {$\theta_2$};
\node (theta3)   at (4,8)    {$\theta_3$};
\node (theta4)   at (4,7)    {$\theta_4$};
\node (theta5)   at (4,6)    {$\theta_5$};
\node (theta6)   at (4,5)    {$\theta_6$};
\node (thetadot) at (4,3.5)    {\dots};
\node (thetaM)   at (4,2.3)  {$\theta_M$};

\path [->] (x1) edge[draw=black] (theta1);
\path [->] (x1) edge[draw=black] (theta2);
\path [->] (x1) edge[draw=black] (theta3);
\path [->] (x2) edge[draw=black] (theta4);
\path [->] (x3) edge[draw=black] (theta5);
\path [->] (x3) edge[draw=black] (theta6);
\path [->] (xp) edge[draw=black] (thetaM);

\draw [decorate, decoration = {brace,mirror}, color=orange]  (4.5,4.8) -- (4.5,10.2) ;

\end{tikzpicture}
\caption{Generic notations.}
\label{fig:generalnotations}
\end{subfigure}
\begin{subfigure}{0.48\linewidth}

\begin{tikzpicture}[scale=0.7]

\draw (-0.9,10.5) --(-1,10.5) -- (-1,6) -- (-0.9,6) ;
\draw (0.9,10.5) --(1,10.5)   -- (1,6)  -- (0.9,6) ;

\node (blue)  at (0,10) {\small $is_{pants}$};
\node (pink)  at (0,9)  {\small $is_{shirt}$};
\node (Paris)     at (0,7)  {\small $b$};

\node (color) at (-2,9.5)   {\small $cat$};
\node (store) at (-2,7) {\tiny $intercept$};

\node (X)     at (-1.9,8) {\Large $X = $};
\node (Theta) at (5.7,8)  {\Large $ = \theta $};

\draw [decorate, decoration = {brace,mirror}] (-1.1,10.2) -- (-1.1,8.8);

\draw (3.1,10.5) --(3,10.5) -- (3,6) -- (3.1,6) ;
\draw (4.9,10.5) --(5,10.5) -- (5,6) -- (4.9,6) ;

\node (mublue)   at (4,10)   {$a_{pants}$};
\node (mupink)   at (4,9)    {$a_{shirt}$};

\node (muparis)  at (4,7)    {\small $b$};

\path [->] (blue)   edge[draw=black] (mublue);
\path [->] (pink)   edge[draw=black] (mupink);
\path [->] (Paris)  edge[draw=black] (muparis);
\end{tikzpicture}
\caption{Notations applied to Relational linear regression}
\label{fig:notationsModel1}
\end{subfigure}
\caption{One-hot encoding for categorical models notations. The parameter $\theta$ is not entirely dependent on each observation. The arrows serve to summarize the interpretability of the model.}
\label{fig:notations}
\end{figure*}
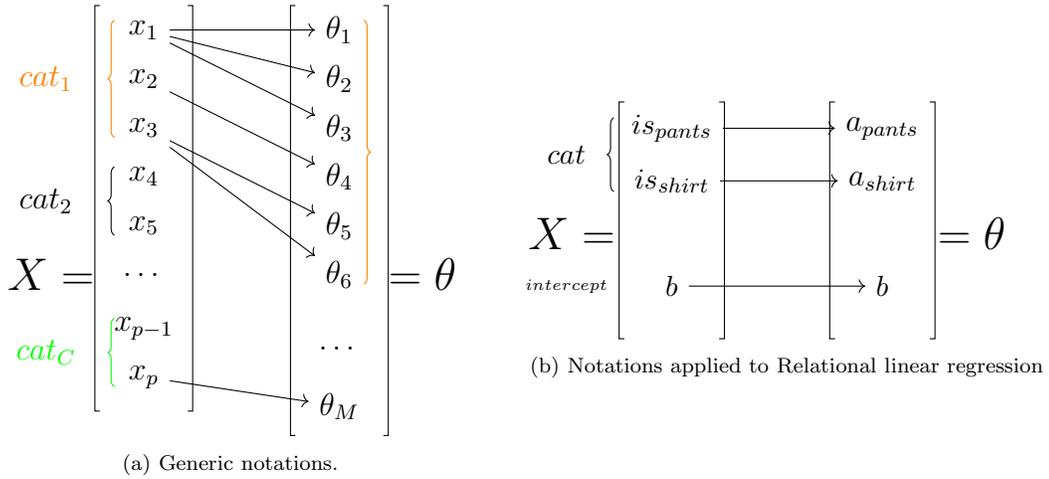

We aim to find the best parameter $\theta^\star \in \mathbb{R}^m$ ($m \geq p$) to minimize the loss $F_{\theta}$ on the whole dataset. Figures \ref{fig:notations} illustrates the formal definition. In Figure \ref{fig:generalnotations}, $\{ \theta_i  \}_{i = 1..6}$ are related to the first feature $\textcolor{orange}{cat_1}$. On relational linear regression \ref{catModel}, such notations give Figure \ref{fig:notationsModel1} where the slope is shared among the category while the intercept is shared by all the observations. 
\begin{multicols}{2}
  \begin{align*}
    f: \quad &\Theta \times \mathcal{Z} \longrightarrow \mathbb{R}\\
            &\theta, (X,y) \longrightarrow  f_{\theta}(X,y)\\
\end{align*}\break
\begin{align*}
\theta^\star = \underset{\theta}{\argmin} \quad F_\theta &= \underset{\theta}{\argmin} \sum_{X,y \in \mathcal{Z}} f_\theta (X,y)\\
             &= \underset{\theta}{\argmin} \sum_{i=1\dots n} f_\theta (X_i,y_i)
\end{align*}
\end{multicols}

\subsection{Gradient descent issues with categorical features}\label{sec:motication}
In classical stochastic gradient descent, an unbiased estimator of the gradient is used. Instead of computing the complete gradient on all observations, the observations are divided into \textit{batches} and the gradient is estimated on them:

\begin{equation}\label{eq:trad-grad}
\nabla_\theta F =  \frac{1}{n} \sum_{obs} \nabla_\theta f_{obs}  = \frac{1}{n} \sum_{i \leq n} \nabla_{\theta} f(X_i, y_i) = \frac{1}{n} \sum_{batch} \quad \sum_{obs \in batch} \nabla_\theta f_{obs}
\end{equation}

In large datasets, it is computationally expensive to compute the exact gradient. To address this, the gradient is estimated by observation batches. Stochastic gradient descent methods are effective in finding the minimum notably because the gradient estimator on a batch is unbiased, as proven by \cite{BachProof}. Regarding categorical models and one-hot-encoding, we stress that categorical parameters are not equally concerned by the batch, especially when the batch is made of a single observation. For example in Model \ref{catModel} $a_{pants}$ is only used on observations that concern a pants product. By construction via one-hot-encoding, each observation concerns one and only one symbol for each categorical attribute. It would be rational to solely update the parameters of the concerned symbol, whereas Equation \ref{eq:trad-grad} computes all the gradient components. In this instance, the gradient $\nabla_{a_{pants}}F$ reduces to:

\begin{equation}\label{eq:gradientDecomp}
\nabla_{a_{pants}} F = \frac{1}{n} \sum_{obs} \nabla_{a_{pants}} f_{obs} = \frac{1}{n} \sum_{batch} \underset{cat(obs) = pants}{\sum_{{obs \in batch}}} \nabla_{a_{pants}} f_{obs}
\end{equation}

What would be the parameter's gradient of a symbol that is not present at all in the dataset? What would be the gradient of $\mu_{hat}$ in Model \ref{catModel} with no hat products in the batch? The set $\{ obs \in batch | cat(obs) = pants \}$ from Equation \ref{eq:gradientDecomp} might be empty. In this case, the parameters related to the \textit{pants} symbol are not active in the batch while they are assigned a zero gradient. Yet, an undefined gradient is not equivalent to a zero-gradient, as shown in Equations \ref{eq:muEncoding}. Thanks to one-hot-encoding, we have prior information about the gradient: if we encounter an observation that does not involve the symbol $s_k$, we know with certainty that the gradient of its related parameters does not exist whereas it is numerically zero in standard implementations. This numerically zero gradient does not convey any information and it should not be used for parameters updates. This atomic property of categorical attributes is completely ignored when using standard SGD approaches. Notice, that this problem is further amplified among successive batches when some improved DL optimization operators are introduced such as gradient with momentum. In this case the structural zeros of categorical data are broadcast among successive batches, leading to biased gradient estimation.

\begin{align*}
    a_{cat} = a_{pants} \times &is_{pants} + a_{shirt} \times is_{shirt} \numberthis \label{eq:muEncoding} \\
    \frac{\partial a_{cat}}{\partial a_{pants}}\mid_{cat = shirt} = \emptyset &\quad ; \quad \frac{\partial a_{cat}}{\partial a_{shirt}}\mid_{cat = pants} = \emptyset \\
\end{align*}
This issue  especially concerns under-represented symbols and small batches. The smaller the cardinality of the symbols and the batch size, the higher the probability of the symbol not being present in the batch. When a symbol is not present in the batch, we state that its related parameters should not be modified. The encoding of categorical data should not be included in the gradient-exposed portion of the model, as it should not affect the model's parameter updates.

\subsection{Convergence guarantees of stochastic gradient descent}\label{sec:BachBottou}

In \cite{BachProof} the authors have unified adaptive optimizers such as AdaGrad \cite{Adagrad} or Adam \cite{adam} and proved their convergence properties. They demonstrated that the exact value of the gradient $\nabla_\theta F$ is not required; instead, an unbiased stochastic estimation $\nabla_\theta f$ is sufficient. To achieve convergence with these optimizers, three assumptions are needed regarding the function being minimized. Firstly, $F$  must be bounded below. Secondly, the norm of $\nabla_\theta f$ must be bounded above by an affine function of the norm of $\nabla_\theta F$. Note that it is not the strongest version of guarantees offered by \cite{BachProof}. Thirdly, $\nabla_\theta F$ must be L-Liptchitz-continuous with respect to the $l_2$-norm:
\begin{equation}\label{eq:assumption-3}
    \forall \quad \theta, \theta' ; \quad \norm{\nabla F (\theta) - \nabla F (\theta')}_2 \leq L\norm{\theta-\theta'}_2
\end{equation}

Not all models meet these assumptions, but both the traditional linear regression that minimizes the mean squared error and its relational version satisfy them. We have $F_{\vec{a}, b} = \sum_{x_i, c_i, y_i} (a_{c_i} x_i + b - y_i)^2$ which is positive. Then for the second one the decomposition of the expected gradient $\nabla_\theta F$ as a sum satisfies it. Finally the third one is also verified, with $m_x = \underset{i}{\max} \quad x_i$
\begin{align*}
    \forall \quad \vec{a},b, \vec{a}',b' ; \quad \pdv{F(\vec{a},b)}{b}  - \pdv{F(\vec{a}',b')}{b} &= 2 \sum_i (\vec{a}_{c_i} - \vec{a}_{c_i}')x_i + (b-b')\\
    \norm{\nabla_b F (\vec{a},b) - \nabla_b F (\vec{a}',b')}_2 &\leq 2n m_x \times (\norm{\vec{a} - \vec{a}'}_2 + \norm{b-b}_2)
\end{align*}
The same logic applies with $\nabla_a F$. Thus the relational version of the linear regression converges if optimized with one version of the adaptive optimizers of \cite{BachProof}.

\section{Solution: gradient estimator for categorical features}

The problem with stochastic gradient descent on one-hot-encoded categorical data arises from the updating of all parameters at each iteration. To address this issue, we propose a new approach that combines a modification of the training loss with a novel gradient estimator. This gradient estimator has been specifically designed to handle categorical data. By combining these two elements, the solution provides a more effective and efficient way of training models on categorical data. The experimentation results show the benefits of this new approach, which has the potential to significantly improve the performance of gradient-based machine learning models on categorical data. All the following is based on the observation that if $\{symbol(obs) = s_k / obs \in batch\}$ is empty, parameters related to the $s_k$ symbol should \textbf{not} impact the parameters update in any way. The proposed gradient estimator makes the difference between a zero gradient and a undefined gradient. Then one needs to count each symbol occurrence and to apply the unbiased gradient estimator. Therefore, one needs to divide the accumulated gradient by the cardinality of $S_k = \{obs \in batch / symbol(obs) = s_k \}$. If this set is empty, parameters related to the $s_k$ symbol should not be updated. This is presented in Algorithm \ref{alg:DivideByTheGood}. Note that this quantity varies at each iteration for every symbol of every categorical parameter.

To support this method, we modify the loss function itself to mirror what we truly aim to minimize while working on categorical data. It gives the following loss:
\begin{equation}\label{eq:realLoss}
    \Tilde{F}_{\theta} = \frac{1}{p} \summ{k=1}{p} \sum_{X,y \in S_k}  \frac{1}{\# S_k} f_{\theta} (X, y)
\end{equation}

With this loss objective function, randomly sampling from $\mathcal{Z}$ and simply summing the gradients of the parameters no longer results in an unbiased gradient estimator. It is necessary to calculate the number of terms contributing to the gradient estimator for each symbol of each categorical parameter. The  $\Tilde{F}_{\binom{\mathcal{Z}}{m} \theta}$ from Equation \ref{eq:unbiased-estimator} properly does it.

\begin{equation}\label{eq:unbiased-estimator}
    \Tilde{F}_{\binom{\mathcal{Z}}{m} \theta} = \frac{1}{p} \summ{k=1}{p} \sum_{X,y \in S_k \cap \binom{\mathcal{Z}}{m}}  \frac{1}{\#[ S_k \cap \binom{\mathcal{Z}}{m}]} f_{\theta} (X, y)
\end{equation}

$\Tilde{F}_{\binom{\mathcal{Z}}{m} \theta}$ is a random estimator of $\Tilde{F}_{\theta}$ where $m$ observations (over the $\mathcal{Z}$) are uniformly drawn. It is the gradient estimator for categorical features (GCE) used by Algorithm \ref{alg:DivideByTheGood}. This estimator is unbiased, proof can be found in \ref{theorem:unbiased}:

\begin{equation*}
    \esp{\nabla \Tilde{F}_{\binom{\mathcal{Z}}{m}}} = \nabla \Tilde{F}_{\theta}
\end{equation*}

This is a sufficient condition for convergence in the previously presented setting Section \ref{sec:BachBottou} as soon as the target loss satisfies regularity conditions. The loss function depicted in Equation \ref{eq:realLoss} seems similar to loss used for classification with unbalanced \textit{output} categories. Let's recall that what we propose here is different as we consider unbalanced \textit{input} symbols. In the case where symbol groups have the same size $C$ then the objective function resumes to $\Tilde{F}_{\theta}$:

\begin{align*}
    \Tilde{F}_{\theta} &= \frac{1}{p}            \summ{k=1}{p} \sum_{X,y \in s_k} \frac{1}{\# s_k} f_{\theta} (X, y)\\
               &= \frac{1}{p}            \summ{k=1}{p}  \frac{1}{C} \sum_{X,y \in s_k}       f_{\theta} (X, y)\\
               &= \frac{1}{p \times C}   \sum_{X,y \in \mathcal{Z}}                                      f_{\theta} (X, y)\\
                &= \frac{1}{\#\mathcal{Z}} \sum_{X,y \in \mathcal{Z}}                                      f_{\theta} (X, y) \\
                &= F_\theta
\end{align*}

as the $p$ symbol groups form  a partition of $\mathcal{Z}$. In this case, our proposed gradient estimator is proportional to the classic one. If one uses the vanilla optimizer, GCE is equivalent to the classic one with a bigger learning rate: 
\begin{equation*}
    \theta_t = \theta_{t-1} - \alpha g_t
\end{equation*}

In this scenario, the gradient's scale is directly related to the learning rate. However, this relationship does not hold true for adaptive optimizers which are highly dependent on the learning rate. In the case of Adam, the update parameter is approximately bounded by the learning rate, making the scale transfer irrelevant. 

\begin{algorithm}
\caption{\textbf{gradient estimator for categorical features}}\label{alg:DivideByTheGood}
\begin{algorithmic}[5]
                
\Require $\mathcal{Z}$: data
\Require $update(\cdot  , \cdot  )$: chosen optimizer
\Require $\theta_0$: Initial parameter vector 
\State $t \gets 0$
\While{$\theta_t$ not converged}
    $t \gets t + 1$
    \State Divide $Z$ in $Batches$
    \For{batch $\in$ Batches}
        \For{symbol $\in$ Alphabet}    
            \State $c_{symbol} \gets 0$
        \EndFor
        \State $\textbf{g} \gets \vec{0}$
        \For{X, y $\in$ batch}
            \State $c_{symbol(X)} \gets c_{symbol(X)} + 1$
            \State Compute $\nabla_{\theta_{t-1}} f_{\theta_{t-1}}(X)$ thanks to $y$
            \State $\textbf{g} \gets \textbf{g} + \nabla_{\theta_{t-1}} f_{\theta_{t-1}}(X)$ \Comment{\textbf{\textcolor{blue}{\small accumulate gradient}}}
        \EndFor
        \State $\theta_{t} \gets \theta_{t-1}$
        \For{$symbol \in Alphabet$}
            \If{$c_{symbol} > 0 $}  \Comment{\textcolor{blue}{\small \textbf{a undefined gradient is not a zero-gradient}}}
                \State $\theta_{t, symbol} \gets update(\theta_{t-1, symbol}, \frac{1}{c_{symbol}}\textbf{g}_{symbol})$ \Comment{\textbf{\textcolor{blue}{\small scaled gradient}}}
            \EndIf
        \EndFor
    \EndFor 
\EndWhile
\end{algorithmic}
\end{algorithm}

\subsection{GCE on relational linear regression}
Let's consider the data from Table \ref{tab:catData} and compare the value of the gradient after the first iteration with the classical gradient estimator and GCE. Lets consider a  a batchsize of 3, so the first iteration concerns the 3 first line of the table, noted $\{(x_1, y_1),(x_2, y_2),(x_3,y_3)\}$. With $\Tilde{g}_{\theta}$ the estimated gradient of $\Tilde{F}_\theta$ with the GCE method and $g_{\theta}$ the classical one of $F_\theta$:
\begin{align*}\label{eq:GCE-nodiff}
    g_{b}  = \Tilde{g}_{b} &= \frac{1}{3}( \nabla_b f(x1, y1) +  \nabla_b f(x2, y2) +  \nabla_b f(x3, y3)) \\ 
    g_{a_{pants}} &= \frac{1}{3}( \nabla_{a_{pants}} f(x1, y1) + 0 +  0)\\
    \Tilde{g}_{a_{pants}} &= \frac{1}{1}( \nabla_{a_{pants}} f(x1, y1))\\
    g_{a_{shirt}} &= \frac{1}{3}( 0 + \nabla_{a_{shirt}} f(x2, y2) + \nabla_{a_{shirt}} f(x3, y3))\\
    \Tilde{g}_{a_{shirt}} &= \frac{1}{2}( \nabla_{a_{shirt}} f(x2, y2) + \nabla_{a_{shirt}} f(x3, y3) )\\
    g_{a_{hat}} &= 0 \quad \text{but} \quad \Tilde{g}_{a_{hat}} = \emptyset
\end{align*}

This very simple example with only one categorical attribute with a two element alphabet highlights the specificity of our proposed gradient estimator. As spotted by the equality $g_{b}  = \Tilde{g}_{b}$, if the parameter is not considered as categorical, this does not change anything to its gradient estimation. The difference between $g_\theta$ and $\Tilde{g}_\theta$  have a bigger impact on the parameter updates when there are multiple categorical parameters.

\section{Experimental and Results}\label{Results}

We have implemented Algorithm \ref{alg:DivideByTheGood} in two different scenarios and programming languages: DL models and categorical models both using one-hot-encoded categorical data. In both cases, we aim to assess the impact of GCE. To evaluate its effectiveness, we compare its performance with the current treatment of categorical parameters in batch gradient descent. We use public datasets listed in Table \ref{tab:datasetChar} as well as a private dataset from the supply chain domain for our evaluations. The chosen metrics for evaluating performance are the mean squared error (MSE) for regression tasks and error rate (i.e., $1 - Accuracy$) for classification tasks.

\begin{table}[h!] 
  \caption{Datasets characteristics}
  \label{tab:datasetChar}
  \begin{footnotesize}
  \begin{center}
  \begin{tabular}{l|rrrrrrr}
    \textbf{Dataset} & Chicago & ACI & compas & DGK  & Forest Cover & KDD99 & UsedCars \\
    \midrule
    instances        & 194m    & 48k & 7.2k   & 72k  & 15k          & 494k  & 38k      \\
    \midrule 
    max cardinality  & 7.9k    & 42  & 341    & 1k & 40             & 66    & 1.1k     \\ 
    \bottomrule		
\end{tabular}
\end{center}
\end{footnotesize}
\end{table}

\subsection{Deep Learning}
In our study, we utilized PyTorch \cite{pytorch} for implementing our proposed solution for DL models. The framework provides ease in updating the gradient of every parameter using Algorithm \ref{alg:DivideByTheGood}. The code and the corresponding experiments can be accessed through the GitHub repository \footnote{\url{https://github.com/ppmdatix/GCE}}. To evaluate the effectiveness of our solution, we conducted experiments on six different datasets with categorical data:

Adult Census Income (ACI) dataset \cite{incomeDataset} aims to predict the wealth status of individuals, Compas dataset predicts the likelihood of re-offending among criminal defendants, Forest Cover dataset \cite{ForestCover} predicts the forest cover type based on categorical characteristics of $30m^2$ forest cells, KDD99 dataset \cite{KDD99} aims at predicting cyber-attacks, Don't Get Kicked (DGK) dataset \cite{DGK} predicts whether a car purchased at auction is a good or a bad buy. Used Cars datasets from Belarus is presented in \ref{subsec:public}.

In order to only measure the impact of GCE, we only use those categorical variables in our experiments. Those dataset tasks are quite easy. As a consequence we use small networks to highlight our approach. The MLP network is made up of 3 dense layers of sizes $[4,8,4]$. We also perform experiments on a ResNet-like network very similar to \cite{RevisitingDeepForTabular}.
We have tested three different optimizers with their default settings: SGD (vanilla), AdaGrad and Adam. Tests have been run on several batch sizes: $2^{5 \dots 10}$. To record results, each experiment has been run 10 times. Results are reproducible in the repository and are recorded in Tables  \ref{tab:resultsMLP32} \ref{tab:resultsRESNET32} \ref{tab:resultsMLP64} \ref{tab:resultsRESNET64} \ref{tab:resultsMLP128}  \ref{tab:resultsRESNET128} \ref{tab:resultsMLP256} \ref{tab:resultsRESNET256} \ref{tab:resultsMLP512} \ref{tab:resultsRESNET512} \ref{tab:resultsMLP1024}     \ref{tab:resultsRESNET1024} in the  \ref{sec:app-DL}. In our experiments, we found that the use of GCE resulted in improvement in loss on the test dataset. Figure \ref{fig:ACIresults} presents the performance of GCE on the Adult Census Income (ACI) dataset. The bigger the batch, the less GCE outperforms the classical estimator. This is logical as in big batches, more symbols are concerned. This proves the need to specifically handle stochastic gradients on categorical data. Results in different settings demonstrate the advantage of using GCE whatever the optimizer. For instance, while AdaGrad has been designed to handle gradients on sparse data (including one-hot encoded data), the use of GCE still resulted in a clear improvement in performance. 
It is noteworthy that our experiments utilized compact network architectures and solely concentrated on the categorical characteristics of the dataset. This was done to isolate the impact of GCE, thereby excluding input variables such as "age" or "income" on the ACI dataset. Despite these stringent limitations, our approach achieved an accuracy of $83\%$ (as shown in Table \ref{tab:resultsRESNET64}) on this dataset when employing GCE. This result is comparable to the state-of-the-art, as reported in \cite{DeepTabularSurvey}. However, only boosting methods have exceeded $87\%$ accuracy, and they have employed all the features, including the non-categorical ones.

\begin{figure}[!ht]
  \centering
  \includegraphics[width=0.7\linewidth]{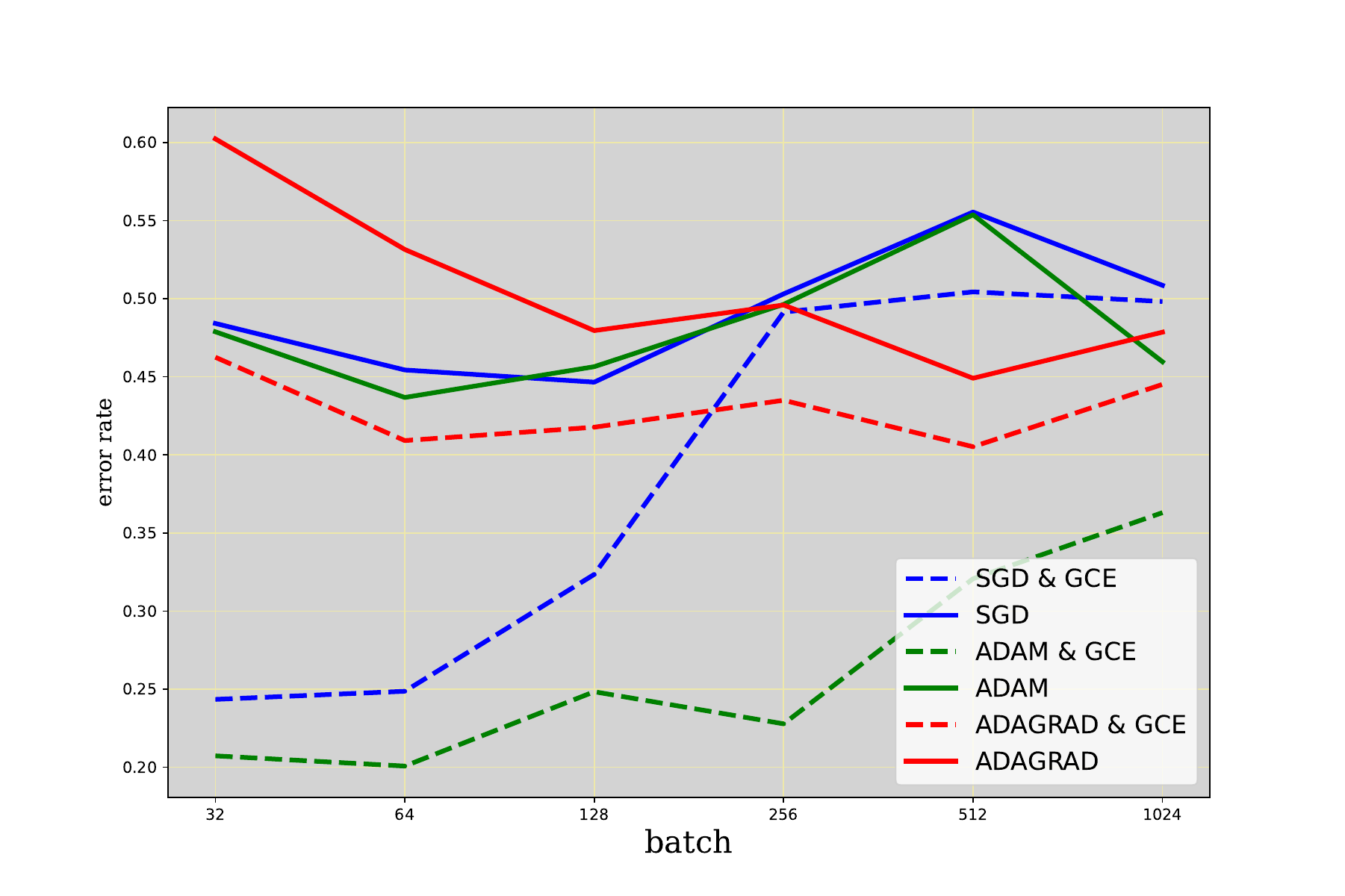}
  \caption{Results (error rate) on the ACI dataset with the MLP network. The dashed curves represent experiments with \textbf{GCE} and show an improvement on the loss for every optimizer used. As the batch size increases, the difference between GCE and  the traditional estimator decreases. This is a logical result because larger batches have a higher probability of including each category.}
  \label{fig:ACIresults}
\end{figure}

\subsection{Categorical models on public datasets}\label{subsec:public}
The experiments for categorical models were conducted using the Envision Domain Specific Language for Supply Chain, a Python-like implementation of SQL designed for supply chain problems. This language includes a differentiable programming layer as described in \cite{Peseux2021DifferentiatingRQ} that provides access to the gradients of categorical models. Stochastic optimization on a relational linear regression with Adam were compared on two publicly available datasets: the Chicago Taxi ride dataset \cite{Chicago} and the Belarus used car dataset \cite{UsedCars}.

For each ride of the Chicago Taxi dataset, we use the taxi identifier, distance, payment type and the tip amount. We use an extended version of the relational linear regression to predict the tip based on the trip distance and the payment type. 
The slope depends on the taxi and the payment method, the intercept remains shared among all the trips, as presented in Equation \ref{eq:otherModel}:
\begin{equation}\label{eq:otherModel}
    \hat{tips} = (\gamma_{\text{taxi}} \times \mu_{\text{payment}}) \times distance + b 
\end{equation}
There is one $\gamma$ per taxi and also one $\mu$ per payment method, that would fit the presented setting with the $Taxi \times Payment$ cross vector construction. As the intercept is shared among all taxis, the dataset is unsplittable although a model based on Equation \ref{eq:splitable}
\begin{equation}\label{eq:splitable}
\hat{tips} = \gamma_{\text{taxi}}  \times distance + b_{taxi}
\end{equation}
could be split into different datasets (one per taxi) and thus we would be in the classical setting of a linear regression.

We also worked on the Belarus used cars dataset. It contains vehicle attributes. We take into account the car manufacturer, the production year, the origin region of the car to predict the selling price of the car as presented in Equation \ref{eq:otherModelUsedCars}.
\begin{equation}\label{eq:otherModelUsedCars}
    \hat{price} = (\gamma_{\text{manufacturer}} \times \mu_{\text{region}}) \times year + b 
\end{equation}

As seen in Table \ref{tab:envisionResult}, the relational batch performed better with our proposition based on Algorithm \ref{alg:DivideByTheGood} with the following setting: 30 epochs ; optimizer Adam with default setting ; batch size of 1. Experiment was reproduced 20 times. 
\begin{table}[h!]
  \caption{Results (RMSE) with categorical models }
  \label{tab:envisionResult}
  \begin{footnotesize}
  \begin{center}
  \begin{tabular}{l|cc}
    \toprule
    Dataset      & Adam         & Adam \& GCE       \\
    \midrule                                                                                     
    Chicago Ride & 35.58 $\pm$ 1.11 & \textbf{9.45 $\pm$ 1.63} \\
    Used Cars & 7.10 $\pm$ 2.45 & \textbf{0.08 $\pm$ 0.01} \\
  \bottomrule
\end{tabular}
\end{center}
\end{footnotesize}
\end{table}


\subsection{Categorical models on a real case}
We have successfully deployed to production such categorical model at Lokad, a French company specialized in supply chain optimization, in order to weekly forecast sales of Celio, a large retail company. 
The dataset is open sourced (with anonymization) and presented \footnote{soon}. The dataset contains 3 years of history and concerns $100k$ different items. The dataset contains multiple categorical attributes for each item.
The objective is to forecast sales at the item level. The implemented categorical model is similar\footnote{we do not disclose the actual model for confidentiality reasons.} to the following:

\begin{align*}\label{eq:celioModel}
    \hat{y}(item, week) = \quad &\theta_{store(item)} \times \theta_{color(item)} \times \theta_{size(item)} \times\\
      & \Theta [group(item), WeekNumber(week)]
\end{align*}

$\Theta [group(item), WeekNumber(week)]$ is a parameter vector that can be seen as a function that aims to capture the annual seasonality for a given group of items:
\begin{equation*}
    \Theta : Groups \times [| 1, 52|] \longrightarrow \mathbb{R}
\end{equation*}

We employed Adam optimizer with its default values along with GCE and stochastic gradient descent for updating the parameters. The use of GCE results in a significant improvement in the performance of the categorical model as compared to the classical gradient estimator. The testing dataset's final loss, measured in terms of decayed MSE, is about an order of magnitude better with GCE. 
However, it is worth noting that while GCE works well in practice, multiplicative models do not meet the assumptions outlined in Section \ref{sec:BachBottou}. Hence, there are no convergence guarantees, and the third assumption remains unsatisfied (see \ref{sec:app-multi-no-gaurantee}). This is not harmful because  setting from Section \ref{sec:BachBottou} is a sufficient one but not necessary to
observe convergence in practice.

\section{Conclusions}\label{sec:conclusion}

This work focuses on the challenge of using stochastic gradient descent for machine learning on categorical data. In the literature, one-hot-encoding is proposed as a solution for creating interpretable models from categorical data, however, this encoding can result in incorrect gradients and incorrect training results. The gradient estimator proposed in this paper overcomes this issue by stating that an undefined gradient should not be treated as a zero-gradient. This new estimator allows for the correct treatment of categorical data in gradient-based models, including DL. The results of the study, including code and details, are open-sourced and demonstrate the utility of the proposed solution on various datasets, including an in-production supply chain model. This dataset is made publicly available for further evaluation and study.

The main contribution of this work is to shed light on the under-appreciation and neglect of categorical data in both public datasets and machine learning as a whole. Despite their widespread use in many key areas, such as health and supply chain, categorical data have not received adequate attention in the field. By highlighting this issue, we hope to inspire further research and encourage the development of methods that specifically address the unique challenges posed by categorical data. For example, one potential solution could be the use of GCE, which is specifically designed to handle this type of data. As a conclusion, our aim is to bring categorical data to the forefront of machine learning research and spur the development of new techniques that better account for the specific requirements of these data.

\section*{Acknowledgment}
This work was supported by the University of Rouen and the French company Lokad. We would like to thank Gaëtan Delétoille, Antonio Cifonelli and Joannes Vermorel for interesting discussions on the topic. A special thanks to Kevin Baumann and Baptiste Miceli for their help on the Celio dataset.

\color{black}
\bibliography{sample}
\appendix
\newpage
\section{Uniform draw}\label{tirage}

Let $Z$ be a non-empty finite set and $T \subset Z$ also non-empty.

We uniformly draw $m > 0$ elements in $Z$ with replacement. We focus on the first drawing where at least one of the $m$ drawn elements belongs to $T$. We note $\Tilde{K}$ this drawing. Thus:

\begin{align}
    \mathbb{P}(\Tilde{K} = 1) &= 1 - (\frac{\abs{Z}-\abs{T}}{\abs{Z}})^m = P_1\\
    \mathbb{P}(\Tilde{K} = n) &= (1 - P_1)^{n-1} P_1
\end{align}

\begin{theorem}[Stopping time]\label{theorem:stop}
$\esp{\Tilde{K}} = \frac{1}{P_1}$
\end{theorem}

\begin{proof}
\begin{align*}
    \esp{\Tilde{K}} = \sum_{n = 1}^{\infty} n \mathbb{P}(\Tilde{K} = n) &= \sum_{n = 1}^{\infty} n (1 - P_1)^{n-1} P_1\\
    &= \frac{P_1}{1 - P_1} \sum_{n = 1}^{\infty} n  (1 - P_1)^{n} 
\end{align*}

For $ 0 < x < 1$  we get:

\begin{align*}
    \sum_{n = 1}^{\infty} n x^n &= \sum_{n = 1}^{\infty} x \frac{\partial x^n}{\partial x}\\
                                &= x \frac{\partial }{\partial x} \sum_{n = 1}^{\infty} x^n\\
                                &= x \frac{\partial }{\partial x} \sum_{n = 0}^{\infty} x^n\\
                                &= x \frac{\partial }{\partial x} \frac{1}{1-x}\\
                                &= \frac{x}{(1-x)^2}
\end{align*}

Then 

\begin{align*}
    \sum_{n = 1}^{\infty} n \mathbb{P}(\Tilde{K} = n) &= \frac{P_1}{1 - P_1} \frac{1-P_1}{P_{1}^2}\\
    &= \frac{1}{P_1}
\end{align*}
\end{proof}

\begin{remark}
It is the same result if the drawings are done without replacement. The only difference is a higher $P_1$.
\end{remark}
\section{Estimator}

Let $Z$ be a non-empty finite set and $T \subset Z$ also non-empty.

We have a score function $s$ on $T$:

\begin{equation*}
\begin{split}
    s \quad : \quad &T \longrightarrow \mathbb{R}\\
    &t                 \longrightarrow s(t)
\end{split}
\end{equation*}

We aim to estimate 
\begin{equation*}
    s_T = \frac{1}{\abs{T}} \sum_{x \in T} s(x)
\end{equation*}

Let $(M_k)_{k \leq K}$ a serie of $K$ draws uniform with replacement of $m$ elements of $Z$.

\begin{remark}\label{rem:Kp1}
    Thanks to Theorem \ref{theorem:stop} we can ignore the first draws $M_{0}$ such as $M_{0} \cap T = \emptyset$
\end{remark}

One notes

\begin{align*}
    M_k &= (M_k \cap T ) \sqcup (M_k \cap (Z \backslash T)) \\
        &= (M_k^T ) \sqcup (M_k \cap (Z \backslash T))
\end{align*}

\[
   avg(M_k^T) = 
\begin{cases}
     \quad 0 \quad \text{if} \quad M_k^T  = \emptyset\\
    \frac{1}{\abs{M_k^T}} \sum_{x \in M_k^T}s(x) \quad \text{otherwise}
\end{cases}
\]
and

\begin{equation*}
    \bar{K} = \abs{ \{ k \leq K | M_k^T \neq \emptyset \}}
\end{equation*}

Thanks to Remark \ref{rem:Kp1} we have $\bar{K} \geq 1$. Then the proposed estimator is $\hat{a}$:

\begin{equation*}
    \hat{a} = \frac{1}{\bar{K}} \sum_{k = 1}^{K}avg(M_k^T)
\end{equation*}

\begin{theorem}[Unbiased estimator]\label{theorem:unbiased}
$\hat{a}$ is an unbiased estimator of $s_T$
\end{theorem}

\begin{proof}
\begin{align*}
    \esp{\Tilde{a}} &= \frac{1}{\hat{K}} \sum_{\underset{k = 1}{M_k^T \neq \emptyset}}^{K} \frac{1}{\abs{(M_k^T)}} \sum_{x \in M_k^T} \esp{s(x)}\\
    &= \frac{\bar{K}}{\bar{K}} \frac{\abs{M_k^T}}{\abs{M_k^T}} \esp{s_T}\\
    &= s_T
\end{align*}
\end{proof}

\section{No guarantees for multiplicative models}\label{sec:app-multi-no-gaurantee}
Let's consider  $h : \reel^3 \rightarrow \reel$ such as $h(x,y,z) = xyz$. Then

\begin{equation}
    \nabla h (x,y,z)=     \begin{pmatrix}
    yz\\
    xz\\
    xy\\
    \end{pmatrix} 
\end{equation}
Then the difference of the gradient cannot be bounded above by the difference of the parameters, i.e. the gradient is not L-Lipschitz-continuous:

\begin{equation*}
    \norm{\nabla h (a,a,a) - \nabla h (b,b,b)}_{2}^2 = 3 (a^2 -b^2)
      = 3 (a - b)^2 (a+b)^2 = (a+b)^2  \times \norm{ \begin{pmatrix}
    a\\
    a\\
    a\\
    \end{pmatrix}  - \begin{pmatrix}
    b\\
    b\\
    b\\
    \end{pmatrix} }_{2}^2 
\end{equation*}

\newpage
\section{Deep learning results}\label{sec:app-DL}

\begin{table}[h!]
    \begin{footnotesize}
    \begin{center}
    \begin{tabular}{l|cc:cc:cc}
    \toprule
    Dataset               &   SGD           & SGD \& \tecnameAbrv & Adagrad & Adagrad \& \tecnameAbrv & Adam        & Adam \& \tecnameAbrv \\
    \midrule
    \textbf{ACI         } & 0.48 $\pm$ 0.24 & 0.24 $\pm$ 0.01 & 0.48 $\pm$ 0.23 & \textbf{0.21 $\pm$ 0.01} & 0.60 $\pm$ 0.24 & 0.46 $\pm$ 0.22 \\ 
    \textbf{DGK         } & 0.56 $\pm$ 0.35 & \textbf{0.12 $\pm$ 0.01} & 0.60 $\pm$ 0.35 & \textbf{0.12 $\pm$ 0.01} & 0.41 $\pm$ 0.36 & 0.35 $\pm$ 0.35 \\ 
    \textbf{Forest Cover} & 1.98 $\pm$ 0.02 & 1.95 $\pm$ 0.01 & 1.98 $\pm$ 0.02 & \textbf{1.44 $\pm$ 0.04} & 1.98 $\pm$ 0.04 & 1.95 $\pm$ 0.03 \\ 
    \textbf{KDD99       } & 1.75 $\pm$ 0.20 & \textbf{0.93 $\pm$ 0.12} & 1.80 $\pm$ 0.17 & \textbf{0.07 $\pm$ 0.03} & 1.94 $\pm$ 0.19 & 1.63 $\pm$ 0.19 \\ 
    \textbf{Used Cars   } & 1.07 $\pm$ 0.06 & \textbf{0.98 $\pm$ 0.01} & 1.08 $\pm$ 0.07 & \textbf{0.99 $\pm$ 0.01} & 1.10 $\pm$ 0.08 & 1.02 $\pm$ 0.04 \\ 

    \bottomrule
    \end{tabular}
    \caption{Results with mlp and batch of 32 (RMSE)}
    \label{tab:resultsMLP32}
    \end{center}
    \end{footnotesize}
\end{table}

\begin{table}[h!]
    \begin{footnotesize}
    \begin{center}
    \begin{tabular}{l|cc:cc:cc}
    \toprule
    Dataset               &   SGD           & SGD \& \tecnameAbrv & Adagrad & Adagrad \& \tecnameAbrv & Adam        & Adam \& \tecnameAbrv \\
    \midrule
    \textbf{ACI         } & 0.49 $\pm$ 0.10 & \textbf{0.20 $\pm$ 0.01} & 0.48 $\pm$ 0.14 & \textbf{0.19 $\pm$ 0.01} & 0.49 $\pm$ 0.14 & \textbf{0.19 $\pm$ 0.01} \\ 
    \textbf{DGK         } & 0.45 $\pm$ 0.19 & \textbf{0.12 $\pm$ 0.01} & 0.50 $\pm$ 0.20 & \textbf{0.12 $\pm$ 0.01} & 0.52 $\pm$ 0.24 & \textbf{0.14 $\pm$ 0.01} \\ 
    \textbf{Forest Cover} & 2.01 $\pm$ 0.04 & \textbf{1.18 $\pm$ 0.01} & 2.01 $\pm$ 0.04 & \textbf{1.03 $\pm$ 0.01} & 2.00 $\pm$ 0.04 & \textbf{1.05 $\pm$ 0.01} \\ 
    \textbf{KDD99       } & 1.81 $\pm$ 0.10 & \textbf{0.07 $\pm$ 0.01} & 1.84 $\pm$ 0.08 & \textbf{0.01 $\pm$ 0.01} & 1.82 $\pm$ 0.12 & \textbf{0.04 $\pm$ 0.01} \\ 
    \textbf{Used Cars   } & 1.23 $\pm$ 0.10 & \textbf{1.02 $\pm$ 0.01} & 1.20 $\pm$ 0.09 & \textbf{1.04 $\pm$ 0.01} & 1.18 $\pm$ 0.05 & \textbf{1.03 $\pm$ 0.01} \\ 

    \bottomrule
    \end{tabular}
    \caption{Results with resnet and batch of 32 (RMSE)}
    \label{tab:resultsRESNET32}
    \end{center}
    \end{footnotesize}
\end{table}

\begin{table}[h!]
    \begin{footnotesize}
    \begin{center}
    \begin{tabular}{l|cc:cc:cc}
    \toprule
    Dataset               &   SGD           & SGD \& \tecnameAbrv & Adagrad & Adagrad \& \tecnameAbrv & Adam        & Adam \& \tecnameAbrv \\
    \midrule
    \textbf{ACI         } & 0.45 $\pm$ 0.25 & 0.25 $\pm$ 0.02 & 0.44 $\pm$ 0.23 & 0.20 $\pm$ 0.03 & 0.53 $\pm$ 0.23 & 0.41 $\pm$ 0.22 \\ 
    \textbf{DGK         } & 0.32 $\pm$ 0.32 & 0.11 $\pm$ 0.01 & 0.43 $\pm$ 0.38 & 0.12 $\pm$ 0.01 & 0.43 $\pm$ 0.38 & 0.29 $\pm$ 0.30 \\ 
    \textbf{Forest Cover} & 2.00 $\pm$ 0.03 & 1.98 $\pm$ 0.02 & 1.99 $\pm$ 0.03 & \textbf{1.54 $\pm$ 0.06} & 1.99 $\pm$ 0.03 & 1.97 $\pm$ 0.02 \\ 
    \textbf{KDD99       } & 1.84 $\pm$ 0.14 & \textbf{1.32 $\pm$ 0.11} & 1.85 $\pm$ 0.23 & \textbf{0.13 $\pm$ 0.08} & 1.95 $\pm$ 0.19 & 1.74 $\pm$ 0.20 \\ 
    \textbf{Used Cars   } & 1.10 $\pm$ 0.06 & \textbf{1.01 $\pm$ 0.01} & 1.11 $\pm$ 0.09 & 1.01 $\pm$ 0.01 & 1.11 $\pm$ 0.10 & 1.05 $\pm$ 0.06 \\ 

    \bottomrule
    \end{tabular}
    \caption{Results with mlp and batch of 64 (RMSE)}
    \label{tab:resultsMLP64}
    \end{center}
    \end{footnotesize}
\end{table}

\begin{table}[h!]
    \begin{footnotesize}
    \begin{center}
    \begin{tabular}{l|cc:cc:cc}
    \toprule
    Dataset               &   SGD           & SGD \& \tecnameAbrv & Adagrad & Adagrad \& \tecnameAbrv & Adam        & Adam \& \tecnameAbrv \\
    \midrule
    \textbf{ACI         } & 0.54 $\pm$ 0.12 & \textbf{0.21 $\pm$ 0.01} & 0.55 $\pm$ 0.11 & \textbf{0.17 $\pm$ 0.01} & 0.57 $\pm$ 0.14 & \textbf{0.16 $\pm$ 0.01} \\ 
    \textbf{DGK         } & 0.45 $\pm$ 0.12 & \textbf{0.11 $\pm$ 0.01} & 0.52 $\pm$ 0.17 & \textbf{0.12 $\pm$ 0.01} & 0.44 $\pm$ 0.16 & \textbf{0.13 $\pm$ 0.01} \\ 
    \textbf{Forest Cover} & 2.02 $\pm$ 0.05 & \textbf{1.37 $\pm$ 0.03} & 1.99 $\pm$ 0.03 & \textbf{1.02 $\pm$ 0.01} & 2.02 $\pm$ 0.04 & \textbf{1.06 $\pm$ 0.01} \\ 
    \textbf{KDD99       } & 1.81 $\pm$ 0.15 & \textbf{0.12 $\pm$ 0.01} & 1.87 $\pm$ 0.07 & \textbf{0.02 $\pm$ 0.01} & 1.89 $\pm$ 0.09 & \textbf{0.06 $\pm$ 0.01} \\ 
    \textbf{Used Cars   } & 1.13 $\pm$ 0.06 & \textbf{0.99 $\pm$ 0.01} & 1.15 $\pm$ 0.07 & \textbf{1.02 $\pm$ 0.01} & 1.20 $\pm$ 0.17 & \textbf{1.01 $\pm$ 0.01} \\ 

    \bottomrule
    \end{tabular}
    \caption{Results with resnet and batch of 64 (RMSE)}
    \label{tab:resultsRESNET64}
    \end{center}
    \end{footnotesize}
\end{table}

\begin{table}[h!]
    \begin{footnotesize}
    \begin{center}
    \begin{tabular}{l|cc:cc:cc}
    \toprule
    Dataset               &   SGD           & SGD \& \tecnameAbrv & Adagrad & Adagrad \& \tecnameAbrv & Adam        & Adam \& \tecnameAbrv \\
    \midrule
    \textbf{ACI         } & 0.45 $\pm$ 0.22 & 0.32 $\pm$ 0.16 & 0.46 $\pm$ 0.23 & 0.25 $\pm$ 0.02 & 0.48 $\pm$ 0.23 & 0.42 $\pm$ 0.22 \\ 
    \textbf{DGK         } & 0.58 $\pm$ 0.37 & 0.30 $\pm$ 0.30 & 0.58 $\pm$ 0.35 & \textbf{0.12 $\pm$ 0.01} & 0.74 $\pm$ 0.29 & 0.49 $\pm$ 0.30 \\ 
    \textbf{Forest Cover} & 1.98 $\pm$ 0.03 & 1.97 $\pm$ 0.02 & 1.99 $\pm$ 0.03 & \textbf{1.60 $\pm$ 0.07} & 1.99 $\pm$ 0.02 & 1.97 $\pm$ 0.02 \\ 
    \textbf{KDD99       } & 1.80 $\pm$ 0.19 & 1.50 $\pm$ 0.15 & 1.89 $\pm$ 0.19 & \textbf{0.45 $\pm$ 0.47} & 1.80 $\pm$ 0.18 & 1.69 $\pm$ 0.18 \\ 
    \textbf{Used Cars   } & 1.16 $\pm$ 0.09 & \textbf{1.03 $\pm$ 0.02} & 1.12 $\pm$ 0.08 & 1.02 $\pm$ 0.01 & 1.13 $\pm$ 0.11 & 1.08 $\pm$ 0.09 \\ 

    \bottomrule
    \end{tabular}
    \caption{Results with mlp and batch of 128 (RMSE)}
    \label{tab:resultsMLP128}
    \end{center}
    \end{footnotesize}
\end{table}

\begin{table}[h!]
    \begin{footnotesize}
    \begin{center}
    \begin{tabular}{l|cc:cc:cc}
    \toprule
    Dataset               &   SGD           & SGD \& \tecnameAbrv & Adagrad & Adagrad \& \tecnameAbrv & Adam        & Adam \& \tecnameAbrv \\
    \midrule
    \textbf{ACI         } & 0.47 $\pm$ 0.12 & \textbf{0.24 $\pm$ 0.01} & 0.48 $\pm$ 0.15 & \textbf{0.19 $\pm$ 0.01} & 0.56 $\pm$ 0.09 & \textbf{0.19 $\pm$ 0.01} \\ 
    \textbf{DGK         } & 0.50 $\pm$ 0.19 & \textbf{0.11 $\pm$ 0.01} & 0.47 $\pm$ 0.18 & \textbf{0.12 $\pm$ 0.01} & 0.39 $\pm$ 0.13 & \textbf{0.13 $\pm$ 0.01} \\ 
    \textbf{Forest Cover} & 2.00 $\pm$ 0.03 & \textbf{1.59 $\pm$ 0.02} & 2.00 $\pm$ 0.03 & \textbf{1.01 $\pm$ 0.01} & 2.00 $\pm$ 0.02 & \textbf{1.04 $\pm$ 0.01} \\ 
    \textbf{KDD99       } & 1.85 $\pm$ 0.15 & \textbf{0.22 $\pm$ 0.01} & 1.90 $\pm$ 0.11 & \textbf{0.01 $\pm$ 0.01} & 1.82 $\pm$ 0.13 & \textbf{0.08 $\pm$ 0.01} \\ 
    \textbf{Used Cars   } & 1.17 $\pm$ 0.05 & \textbf{1.02 $\pm$ 0.01} & 1.17 $\pm$ 0.04 & \textbf{1.06 $\pm$ 0.02} & 1.16 $\pm$ 0.07 & \textbf{1.03 $\pm$ 0.01} \\ 

    \bottomrule
    \end{tabular}
    \caption{Results with resnet and batch of 128 (RMSE)}
    \label{tab:resultsRESNET128}
    \end{center}
    \end{footnotesize}
\end{table}

\begin{table}[h!]
    \begin{footnotesize}
    \begin{center}
    \begin{tabular}{l|cc:cc:cc}
    \toprule
    Dataset               &   SGD           & SGD \& \tecnameAbrv & Adagrad & Adagrad \& \tecnameAbrv & Adam        & Adam \& \tecnameAbrv \\
    \midrule
    \textbf{ACI         } & 0.50 $\pm$ 0.26 & 0.49 $\pm$ 0.25 & 0.50 $\pm$ 0.23 & \textbf{0.23 $\pm$ 0.01} & 0.50 $\pm$ 0.26 & 0.43 $\pm$ 0.23 \\ 
    \textbf{DGK         } & 0.34 $\pm$ 0.36 & 0.30 $\pm$ 0.31 & 0.53 $\pm$ 0.36 & \textbf{0.11 $\pm$ 0.01} & 0.50 $\pm$ 0.36 & 0.28 $\pm$ 0.29 \\ 
    \textbf{Forest Cover} & 1.98 $\pm$ 0.03 & 1.98 $\pm$ 0.02 & 1.98 $\pm$ 0.02 & \textbf{1.79 $\pm$ 0.06} & 2.00 $\pm$ 0.03 & 1.97 $\pm$ 0.02 \\ 
    \textbf{KDD99       } & 1.84 $\pm$ 0.24 & 1.70 $\pm$ 0.23 & 1.74 $\pm$ 0.10 & \textbf{0.57 $\pm$ 0.33} & 1.88 $\pm$ 0.23 & 1.78 $\pm$ 0.23 \\ 
    \textbf{Used Cars   } & 1.08 $\pm$ 0.07 & 1.04 $\pm$ 0.02 & 1.10 $\pm$ 0.06 & \textbf{1.02 $\pm$ 0.01} & 1.11 $\pm$ 0.07 & 1.07 $\pm$ 0.05 \\ 

    \bottomrule
    \end{tabular}
    \caption{Results with mlp and batch of 256 (RMSE)}
    \label{tab:resultsMLP256}
    \end{center}
    \end{footnotesize}
\end{table}

\begin{table}[h!]
    \begin{footnotesize}
    \begin{center}
    \begin{tabular}{l|cc:cc:cc}
    \toprule
    Dataset               &   SGD           & SGD \& \tecnameAbrv & Adagrad & Adagrad \& \tecnameAbrv & Adam        & Adam \& \tecnameAbrv \\
    \midrule
    \textbf{ACI         } & 0.54 $\pm$ 0.14 & \textbf{0.24 $\pm$ 0.01} & 0.53 $\pm$ 0.13 & \textbf{0.19 $\pm$ 0.01} & 0.52 $\pm$ 0.10 & \textbf{0.19 $\pm$ 0.01} \\ 
    \textbf{DGK         } & 0.59 $\pm$ 0.16 & \textbf{0.11 $\pm$ 0.01} & 0.50 $\pm$ 0.16 & \textbf{0.12 $\pm$ 0.01} & 0.48 $\pm$ 0.19 & \textbf{0.14 $\pm$ 0.01} \\ 
    \textbf{Forest Cover} & 1.99 $\pm$ 0.04 & \textbf{1.73 $\pm$ 0.03} & 2.01 $\pm$ 0.02 & \textbf{1.03 $\pm$ 0.01} & 2.03 $\pm$ 0.02 & \textbf{1.07 $\pm$ 0.01} \\ 
    \textbf{KDD99       } & 1.76 $\pm$ 0.07 & \textbf{0.47 $\pm$ 0.04} & 1.80 $\pm$ 0.05 & \textbf{0.02 $\pm$ 0.01} & 1.86 $\pm$ 0.09 & \textbf{0.16 $\pm$ 0.02} \\ 
    \textbf{Used Cars   } & 1.17 $\pm$ 0.11 & \textbf{1.00 $\pm$ 0.01} & 1.17 $\pm$ 0.11 & 1.06 $\pm$ 0.01 & 1.13 $\pm$ 0.06 & \textbf{1.01 $\pm$ 0.01} \\ 

    \bottomrule
    \end{tabular}
    \caption{Results with resnet and batch of 256 (RMSE)}
    \label{tab:resultsRESNET256}
    \end{center}
    \end{footnotesize}
\end{table}

\begin{table}[h!]
    \begin{footnotesize}
    \begin{center}
    \begin{tabular}{l|cc:cc:cc}
    \toprule
    Dataset               &   SGD           & SGD \& \tecnameAbrv & Adagrad & Adagrad \& \tecnameAbrv & Adam        & Adam \& \tecnameAbrv \\
    \midrule
    \textbf{ACI         } & 0.56 $\pm$ 0.24 & 0.50 $\pm$ 0.24 & 0.55 $\pm$ 0.26 & 0.32 $\pm$ 0.17 & 0.45 $\pm$ 0.26 & 0.41 $\pm$ 0.24 \\ 
    \textbf{DGK         } & 0.54 $\pm$ 0.35 & 0.47 $\pm$ 0.36 & 0.80 $\pm$ 0.23 & \textbf{0.22 $\pm$ 0.11} & 0.51 $\pm$ 0.38 & 0.50 $\pm$ 0.38 \\ 
    \textbf{Forest Cover} & 1.97 $\pm$ 0.02 & 1.97 $\pm$ 0.02 & 2.01 $\pm$ 0.03 & \textbf{1.92 $\pm$ 0.02} & 2.01 $\pm$ 0.03 & 1.99 $\pm$ 0.02 \\ 
    \textbf{KDD99       } & 1.82 $\pm$ 0.17 & 1.74 $\pm$ 0.16 & 1.89 $\pm$ 0.18 & \textbf{1.41 $\pm$ 0.17} & 1.85 $\pm$ 0.20 & 1.79 $\pm$ 0.20 \\ 
    \textbf{Used Cars   } & 1.10 $\pm$ 0.08 & 1.05 $\pm$ 0.04 & 1.10 $\pm$ 0.07 & 0.99 $\pm$ 0.03 & 1.11 $\pm$ 0.11 & 1.08 $\pm$ 0.08 \\ 

    \bottomrule
    \end{tabular}
    \caption{Results with mlp and batch of 512 (RMSE)}
    \label{tab:resultsMLP512}
    \end{center}
    \end{footnotesize}
\end{table}

\begin{table}[h!]
    \begin{footnotesize}
    \begin{center}
    \begin{tabular}{l|cc:cc:cc}
    \toprule
    Dataset               &   SGD           & SGD \& \tecnameAbrv & Adagrad & Adagrad \& \tecnameAbrv & Adam        & Adam \& \tecnameAbrv \\
    \midrule
    \textbf{ACI         } & 0.47 $\pm$ 0.15 & \textbf{0.25 $\pm$ 0.02} & 0.49 $\pm$ 0.11 & \textbf{0.18 $\pm$ 0.01} & 0.54 $\pm$ 0.15 & \textbf{0.19 $\pm$ 0.01} \\ 
    \textbf{DGK         } & 0.52 $\pm$ 0.21 & \textbf{0.13 $\pm$ 0.01} & 0.60 $\pm$ 0.21 & \textbf{0.12 $\pm$ 0.01} & 0.46 $\pm$ 0.17 & \textbf{0.13 $\pm$ 0.01} \\ 
    \textbf{Forest Cover} & 2.03 $\pm$ 0.04 & \textbf{1.86 $\pm$ 0.03} & 2.01 $\pm$ 0.04 & \textbf{1.02 $\pm$ 0.01} & 2.00 $\pm$ 0.05 & \textbf{1.08 $\pm$ 0.01} \\ 
    \textbf{KDD99       } & 1.79 $\pm$ 0.07 & \textbf{0.88 $\pm$ 0.03} & 1.84 $\pm$ 0.10 & \textbf{0.02 $\pm$ 0.01} & 1.84 $\pm$ 0.08 & \textbf{0.22 $\pm$ 0.04} \\ 
    \textbf{Used Cars   } & 1.18 $\pm$ 0.07 & \textbf{1.03 $\pm$ 0.01} & 1.15 $\pm$ 0.08 & \textbf{1.04 $\pm$ 0.01} & 1.14 $\pm$ 0.05 & \textbf{1.02 $\pm$ 0.01} \\ 

    \bottomrule
    \end{tabular}
    \caption{Results with resnet and batch of 512 (RMSE)}
    \label{tab:resultsRESNET512}
    \end{center}
    \end{footnotesize}
\end{table}

\begin{table}[h!]
    \begin{footnotesize}
    \begin{center}
    \begin{tabular}{l|cc:cc:cc}
    \toprule
    Dataset               &   SGD           & SGD \& \tecnameAbrv & Adagrad & Adagrad \& \tecnameAbrv & Adam        & Adam \& \tecnameAbrv \\
    \midrule
    \textbf{ACI         } & 0.51 $\pm$ 0.26 & 0.50 $\pm$ 0.26 & 0.46 $\pm$ 0.26 & 0.36 $\pm$ 0.21 & 0.48 $\pm$ 0.25 & 0.45 $\pm$ 0.25 \\ 
    \textbf{DGK         } & 0.42 $\pm$ 0.37 & 0.42 $\pm$ 0.37 & 0.49 $\pm$ 0.37 & 0.16 $\pm$ 0.06 & 0.59 $\pm$ 0.36 & 0.58 $\pm$ 0.37 \\ 
    \textbf{Forest Cover} & 1.99 $\pm$ 0.03 & 1.99 $\pm$ 0.03 & 1.99 $\pm$ 0.02 & 1.95 $\pm$ 0.01 & 1.99 $\pm$ 0.03 & 1.98 $\pm$ 0.03 \\ 
    \textbf{KDD99       } & 1.83 $\pm$ 0.27 & 1.77 $\pm$ 0.26 & 1.91 $\pm$ 0.29 & 1.67 $\pm$ 0.30 & 1.82 $\pm$ 0.21 & 1.78 $\pm$ 0.20 \\ 
    \textbf{Used Cars   } & 1.08 $\pm$ 0.08 & 1.06 $\pm$ 0.06 & 1.19 $\pm$ 0.13 & 1.08 $\pm$ 0.08 & 1.09 $\pm$ 0.07 & 1.07 $\pm$ 0.06 \\ 

    \bottomrule
    \end{tabular}
    \caption{Results with mlp and batch of 1024 (RMSE)}
    \label{tab:resultsMLP1024}
    \end{center}
    \end{footnotesize}
\end{table}

\begin{table}[h!]
    \begin{footnotesize}
    \begin{center}
    \begin{tabular}{l|cc:cc:cc}
    \toprule
    Dataset               &   SGD           & SGD \& \tecnameAbrv & Adagrad & Adagrad \& \tecnameAbrv & Adam        & Adam \& \tecnameAbrv \\
    \midrule
    \textbf{ACI         } & 0.53 $\pm$ 0.13 & \textbf{0.32 $\pm$ 0.07} & 0.51 $\pm$ 0.10 & \textbf{0.18 $\pm$ 0.01} & 0.54 $\pm$ 0.15 & \textbf{0.19 $\pm$ 0.01} \\ 
    \textbf{DGK         } & 0.48 $\pm$ 0.20 & \textbf{0.18 $\pm$ 0.04} & 0.44 $\pm$ 0.14 & \textbf{0.13 $\pm$ 0.01} & 0.48 $\pm$ 0.19 & \textbf{0.15 $\pm$ 0.01} \\ 
    \textbf{Forest Cover} & 2.00 $\pm$ 0.04 & \textbf{1.88 $\pm$ 0.04} & 2.01 $\pm$ 0.04 & \textbf{1.04 $\pm$ 0.01} & 2.03 $\pm$ 0.04 & \textbf{1.13 $\pm$ 0.01} \\ 
    \textbf{KDD99       } & 1.80 $\pm$ 0.10 & \textbf{1.12 $\pm$ 0.09} & 1.86 $\pm$ 0.10 & \textbf{0.05 $\pm$ 0.01} & 1.81 $\pm$ 0.06 & \textbf{0.32 $\pm$ 0.07} \\ 
    \textbf{Used Cars   } & 1.11 $\pm$ 0.02 & \textbf{1.05 $\pm$ 0.01} & 1.19 $\pm$ 0.09 & \textbf{1.03 $\pm$ 0.01} & 1.12 $\pm$ 0.03 & \textbf{1.01 $\pm$ 0.01} \\ 

    \bottomrule
    \end{tabular}
    \caption{Results with resnet and batch of 1024 (RMSE)}
    \label{tab:resultsRESNET1024}
    \end{center}
    \end{footnotesize}
\end{table}

\end{document}